\documentclass[10pt,twocolumn,letterpaper]{article}

\usepackage[top=0.9 in,bottom=0.9in,left=0.85in,right=0.85in,columnsep=0.4in]{geometry}

\usepackage{amsmath}
\usepackage{amssymb}
\usepackage{amsthm}
\usepackage{color}

\DeclareMathOperator{\EX}{\mathbb{E}}
\newcommand{\Var}{\mathbb{V}{\rm ar}}
\newtheorem{lemma}{Lemma}
\newtheorem{theorem}{Theorem}

\newcommand{\erf}{{\rm erf}}
\newcommand{\hns}{\hspace{-0.025in}}

\title{\textbf{Information-Theoretic Lower Bounds for Zero-Order Stochastic Gradient Estimation}}

\author{
Abdulrahman Alabdulkareem\\
Department of Computer Science\\
Purdue University\\
West Lafayette, IN 47907, USA\\
\texttt{alabdulk@purdue.edu}\\
\and
Jean Honorio\\
Department of Computer Science\\
Purdue University\\
West Lafayette, IN 47907, USA\\
\texttt{jhonorio@purdue.edu}}

\date{}

\begin{document}

\maketitle

\begin{abstract}
In this paper we analyze the necessary number of samples to estimate the gradient of any multidimensional smooth (possibly non-convex) function in a zero-order stochastic oracle model. In this model, an estimator has access to noisy values of the function, in order to produce the estimate of the gradient. We also provide an analysis on the sufficient number of samples for the finite difference method, a classical technique in numerical linear algebra. For $T$ samples and $d$ dimensions, our information-theoretic lower bound is $\Omega(\sqrt{d/T})$. We show that the finite difference method for a bounded-variance oracle has rate $O(d^{4/3}/\sqrt{T})$ for functions with zero third and higher order derivatives. These rates are tight for Gaussian oracles. Thus, the finite difference method is not minimax optimal, and therefore there is space for the development of better gradient estimation methods.
\end{abstract}

\section{Introduction}

Gradient estimation is a fundamental mathematical technique that was known by Euler since the 1700’s. Today, gradient estimation is used in almost all scientific fields in some form or another. In Machine Learning, for example, gradient estimation is used in Reinforcement Learning with policy-gradient methods to estimate the update performed to the policy \cite{Sutton00}.  More generally, performing gradient descent on any loss function which is intractable is the essence of gradient estimation in Machine Learning and is prevalent in the field \cite{schulman00}. Gradient estimation is also widespread in other fields such as optimization, engineering, operations research, among others.

With such a widely used technique arises a fundamental question of the hardness regarding gradient approximation with respect to the dimensionality of the function domain, as well as the number of function evaluations required. Using information theory it is possible to lower bound the theoretical best achievable error that any conceivable gradient estimator can have.
We assume that a gradient estimator has access to a zero-order oracle, that provides a \emph{noisy} value of the function at a given point.
The stochastic nature of the oracle makes the estimation task not only harder, but also amenable to sound statistical analysis.

To the best of our knowledge, no prior work has analyzed the theoretical lower bound on the error of estimating the gradient of a multi-dimensional function. Other works have proposed more efficient gradient estimation methods for sparse gradients with theoretical guarantees \cite{borkar00} or Monte Carlo methods without theoretical guarantees \cite{wieland00}. Some works have used information theory to lower bound the theoretical error on oracle convex optimization, however, the gradient was either provided by the oracle \cite{Agarwal00} or calculated using a fixed number of points \cite{Duchi00}.

Our contributions are summarized as follows. We prove that fundamental theoretical error lower-bound for the general case of gradient estimation on $d$-dimensional functions using $T$ samples, is $\Omega(\sqrt{d/T})$. We show that the finite difference method converges with rate $O(d^{4/3}/\sqrt{T})$ for functions with zero third and higher order derivatives.
The gap between finite differences and our lower bound, suggests that better gradient estimation methods could be developed.

\section{Zero-Order Oracle and Model Definition}

In this section, we define our problem setup. Assume that we are given a set $\mathbb{S} \subset \mathbb{R}^d$, a set of $d$-dimensional functions $\mathcal{F} \subset \{f|f:\mathbb{S}\to\mathbb{R}\}$, and a point $x^* \in \mathbb{S}$ where we want to estimate the gradient $\nabla f(x^*) \in \mathbb{R}^d$ using only $T$ queries to a zero-order stochastic oracle.

A zero-order stochastic oracle is defined to be a random function $\phi : \mathbb{S} \rightarrow \mathbb{R}$ which answers queries as follows. The oracle $\phi$ receives a point $x\in \mathbb{S}$ and returns a noisy unbiased estimate $\phi(x,f)$ of the function $f$ with bounded variance. That is,
\[
\EX(\phi(x, f))=f(x)
   \text{ and }
\Var[\phi(x, f)] \leq \sigma^2
\]

Furthermore, let $\mathcal{O}$ be the set of stochastic zero-order oracles. A model $\mathcal{M}$ is defined to be a function that makes a total of $T$ queries. That is, for $t = 1,\dots,T$, the model M sends $x_t \in S$ to the oracle $\phi \in \mathcal{O}$ and the oracle returns the noisy unbiased estimate $\phi(x_t, f)$. The model then uses all the noisy estimates and outputs a gradient estimate $\nabla \widehat{f}(x^*)_\mathcal{M} \in \mathbb{R}^d$. The oracle $\phi$ cannot be queried more than $T$ times.
Let $\mathbb{M}_T$ be the class of all models as described above.

\subsection{Error Definition and Minimax error}

In this section we will cover the minimax framework and explicitly formulate the theoretical best achievable error in full form.
The minimax framework consists of a well defined objective of revealing information on the theoretical optimality of algorithms and is widely used in statistics and machine learning \cite{wainwright00,wasserman2006all}.

The minimax framework considers a family of distributions over a sample space, which in our case is a class of multi-dimensional functions $\mathcal{F}$ with domain $\mathbb{S}$. The minimax framework also considers a function that is a mapping from the family of distributions to a parameter, which in our case is simply the gradient of a function $\nabla f(x^*)$ for some $f \in \mathcal{F}$ and $x^* \in \mathbb{S}$. We aim to estimate the value of $\nabla f(x^*)$ based on a sequence of queries to a zero-order stochastic oracle $\phi \in \mathcal{O}$. We evaluate the quality of an estimate $\nabla \widehat{f}(x^*)_\mathcal{M}$ of some model (estimator) $\mathcal{M}\in \mathbb{M}_T$ in terms of the expected risk. That is,
\[
\EX_\phi [\Vert \nabla \widehat{f}(x^*)_\mathcal{M} - \nabla f(x^*) \Vert _ 1 ]
\]
The expected value is due to the fact that the oracle $\phi$ is stochastic thus we want to average evaluation of the model (estimator) $\mathcal{M}$.
Next we take a look at the maximum risk of a model $\epsilon(\mathcal{M}, \mathcal{F}, \mathbb{S}, \phi)$. Given a class functions $\mathcal{F}$ with domain $\mathbb{S}$ and an oracle $\phi \in \mathcal{O}$, the maximum risk of a model is simply defined to be the highest (worst) expected risk that model gets over all the different functions $f \in \mathcal{F}$ and different points $x^* \in \mathbb{S}$. More formally,
\begin{equation}\label{epsilondefinition}
\epsilon(\mathcal{M}, \mathcal{F}, \mathbb{S}, \phi) := \sup_{f\in\mathcal{F}} \sup_{x^*\in\mathbb{S}}
\EX_\phi [\Vert \nabla \widehat{f}(x^*)_\mathcal{M} - \nabla f(x^*) \Vert _ 1 ]
\end{equation}

Finally, the minimax framework dictates to use the model (estimator) which has the lowest maximum risk out of all other possible models. That is,
\[
\epsilon^*(\mathcal{F}, \mathbb{S}, \phi) := \inf_{\mathcal{M}\in\mathbb{M}_T}
[\epsilon(\mathcal{M}, \mathcal{F}, \mathbb{S}, \phi)]
\]
which is the minimax error, where the supremum (worst-case) is taken over the class of functions $\mathcal{F}$ and the infimum (best-case) is taken over the class of models $\mathbb{M}_T$.
The main goal of this paper is to find a meaningful lower bound for the minimax error.

\subsection{Finite Difference Method}

As an example of a particular gradient estimator, we analyze the finite difference method (FDM) which is a predominantly used method. More formally, we analyze the upper bound of the error. FDM applied on a $d$-dimensional function $f$ using $T$ queries performs $\frac{T}{d}$ query calls for each dimension. For some small parameter $h>0$, FDM outputs a gradient estimate $\nabla \widehat{f}(x^*)_\mathcal{FDM}=[\widehat{\nabla}_1, \widehat{\nabla}_2,\dots,\widehat{\nabla}_d]^\top$ and for each dimension $i$, the value $\widehat{\nabla}_i$ is calculated as follows:
\[
\widehat{\nabla}_i = \frac{1}{T/2d}
\sum_{j=1}^{T/2d}
{\left(\frac{\phi^{(j)}(x^*+he_i, f)-\phi^{(j)}(x^*-he_i, f)}{2h}\right)}
\]
where $e_i = [0,\dots,0,1,0,\dots,0]^\top$ and the entry equal to $1$ is at the $i$-th index. Thus the method basically makes $T/d$ queries for each of the $d$ dimensions.

\section{Main Results}
In this section, we present our main results. Theorem~\ref{hyperplanes} is the main result of our paper and is the theoretic lower-bound on the error. Our information-theoretic result relies on the construction of a restricted class of functions. The use of restricted ensembles is customary for information-theoretic lower bounds \cite{santhanam00,wang00,tandon00,ke00}.

First, we present our information-theoretic lowerbound on the minimax error which is $\Omega(\sqrt{d/T})$ where $d$ is the number of dimensions and $T$ is the number of samples. This implies that there cannot exist an oracle-querying gradient estimator that achieves a convergence rate faster than $\Omega(\sqrt{d/T})$.
\begin{theorem}
\label{hyperplanes}
For any class $\mathcal{F}$ of differentiable functions in $\mathbb{R}^d$ with input space $\mathbb{S}$ and gradient bounded by K, i.e., $\sup_{x\in\mathbb{S}} \sup_{f\in\mathcal{F}}{\Vert\nabla f(x)\Vert_\infty}\leq K$, there exists a constant $c$ (independent of $T$) such that the minimax error is bounded as
\[
\epsilon^*(\mathcal{F}, \mathbb{S}, \phi) \geq c\sqrt{\frac{d}{T}}
\]
provided that
$
\left(\inf_{f \in \mathcal{F}, x^* \in \mathbb{S}, j \in \{0,\dots,d\}}{|\nabla f(x^*)_j|}\right)\ne0
$
and
$
\left(\sup_{f \in \mathcal{F}, x^* \in \mathbb{S}}{\Vert f(x^*)\Vert_\infty}\right)\leq 2\sigma
$
where $\sigma^2$ is the upper bound on the variance of the oracle.
\end{theorem}

Next we analyze the upper bound on the error for the finite difference method for a bounded-variance oracle, which is $O(d^{4/3}/T^{1/3})$ when the function has non-zero third derivative, and $O(d^{3/2}/\sqrt{T})$ when the function has zero third and higher order derivatives. \begin{theorem}
\label{FDM}
Consider any valid oracle $\phi \in \mathcal{O}$ with upper bounded variance $\sigma^2$, and any set $\mathcal{F}$ of $d$-dimensional functions with third order derivatives upperbounded by K, i.e., $\sup_{i \in \{1,\dots,d\}}\sup_{x\in\mathbb{S}}{|f^{iii}(x)|}\leq K$, where $f^{iii}$ represents the third derivative of $f$ with respect to the $i$-th variable.
If $K>0$, the error for the FDM model with respect to $T \geq 2d$ is upper bounded by
\[
\epsilon(\mathcal{FDM}, \mathcal{F}, \mathbb{S}, \phi) \leq d^{(4/3)} \sqrt[3]{\frac{1}{T}} \sqrt[3]{\frac{9\sigma^2K}{2}}
\]
Otherwise, if $K=0$ then
\[
\epsilon(\mathcal{FDM}, \mathcal{F}, \mathbb{S}, \phi) \leq d^{(3/2)}\frac{1}{\sqrt{T}}\frac{2\sigma}{h_r}
\]
where $h_r$ is a constant that depends on the point $x^*$ and the input space $\mathbb{S}$.
\end{theorem}

In what follows, we show that the above rates are \emph{tight} for Gaussian oracles.
Thus, the finite difference method is not minimax optimal, and therefore there is space for the development of better gradient estimation methods.
\begin{theorem}
\label{FDMgaussian}
Consider a Gaussian oracle $\phi \in \mathcal{O}$ with variance $\sigma^2$, and any set $\mathcal{F}$ of $d$-dimensional functions with third order derivatives upperbounded by K, i.e., $\sup_{i \in \{1,\dots,d\}}\sup_{x\in\mathbb{S}}{|f^{iii}(x)|}\leq K$, where $f^{iii}$ represents the third derivative of $f$ with respect to the $i$-th variable.
If $K>0$, the error for the FDM model with respect to $T \geq 2d$ is upper bounded by
\[
\begin{split}
&\epsilon(\mathcal{FDM}, \mathcal{F}, \mathbb{S}, \phi) \\
&= d^{(4/3)} \sqrt[3]{\frac{1}{T}} \sqrt[3]{\frac{\sigma^2K}{12\pi}} \left(2\exp\hns\left(-\frac{\pi}{4}\right)+\erf\hns\left(\frac{1}{2\sqrt{\pi}}\right)\right) c
\end{split}
\]
where $\erf$ is the Gauss error function and $c \in [0,1]$.
Otherwise, if $K=0$ then
\[
\epsilon(\mathcal{FDM}, \mathcal{F}, \mathbb{S}, \phi) = d^{(3/2)}\frac{1}{\sqrt{T}}\frac{\sigma}{h_r} \sqrt{\frac{2}{\pi}}
\]
where $h_r$ is a constant that depends on the point $x^*$ and the input space $\mathbb{S}$.
\end{theorem}

\section{Proof of Theorem~\ref{hyperplanes}: Information-Theoretic Lower Bound}
In this section we provide the detailed proof for Theorem~\ref{hyperplanes}. The sketch of the proof is as following. We start by defining a set of sum of ensemble functions that is parameterized by a discrete set. Then, we show that approximating the gradient is as hard as recovering the discrete-valued parameters. Then, we define a specific zero-order stochastic oracle and analyze its related KL divergence. Finally, we use Fano's inequality and combine all the results to prove our claim.
\subsection{Defining Function Space}

Define $\mathcal{V}\subseteq\{-1,+1\}^d$ such that
for any $\alpha,\beta\in\mathcal{V}$, if $\alpha\ne\beta$ then
\begin{equation}\label{defV}
\sum_{i=1}^{d}1[\alpha_i\ne\beta_i]\geq d/4
\end{equation}
Then using a binomial counting construction from \cite{Matousek00}, it is possible to construct a set $\mathcal{V}$ with cardinality
\begin{equation}\label{cardinalityV}
|\mathcal{V}|\geq(2/\sqrt{e})^{d/2}
\end{equation}

Now we want to define a set of sum of ensemble functions $\mathcal{G}(\delta, h)$ to be the base class of functions that is parameterized by $\alpha$ where $\delta \in \mathbb{R}$ and $h_i:\mathbb{R} \rightarrow \mathbb{R}$. Note that $\mathcal{G} \subset \mathcal{F}$. Thus, it is sufficient to prove the lowerbound for the subset $\mathcal{G}$ since this implies a lower bound over the set $\mathcal{F}$.
Let $g_\alpha \in \mathcal{G}(\delta, h)$ be defined as
\begin{equation}\label{galpha}
g_\alpha(x):=\frac{\delta}{d}\sum_{i=1}^{d}\alpha_i h_i(x(i))
=\frac{\delta}{d}\langle \alpha, H(x) \rangle
\end{equation}
where $x(i)$ refers to the $i$-th coordinate of $x$, $\langle \cdot , \cdot \rangle$ denotes the dot product, and $H(x)=[h_1(x(1)), h_2(x(2)), \dots, h_d(x(d))]^\top$.
Thus $\nabla g_\alpha(x) = \frac{\delta}{d} (\alpha \odot \nabla H(x))$, where the operator $\odot$ refers to the Hadamard product. The term $\delta$ takes into account, for a fixed $h$, how ``close" we want our gradients for different functions in the base class.

\subsection{Minimum Distance between Functions in the Class}

Let $\psi(\mathcal{G}(\delta, h))$ denote the minimum discrepancy in the $\Vert \cdot \Vert _1$ norm between the gradient at the minimum point $x^* \in \mathbb{S}$ of any two functions in $\mathcal{G}(\delta, h)$. Let $\alpha, \beta \in \mathcal{V}$ where $\alpha \ne \beta$. We formally define the discrepancy as

\begin{equation}\label{psiDefinition}
\psi(\mathcal{G}(\delta, h))
:=\inf_{\alpha,\beta\in\mathcal{V}}\inf_{x^* \in \mathbb{S}}{\Vert \nabla g_\alpha(x^*)-\nabla g_\beta(x^*)\Vert_1}
\end{equation}
Note that
\[
\begin{split}
\psi(\mathcal{G}(\delta, h))=&\inf_{\alpha,\beta\in\mathcal{V}}\inf_{x^* \in \mathbb{S}}
\frac{\delta}{d}\Vert (\alpha - \beta)\odot \nabla H(x^*) \Vert_1
\\=&\inf_{\alpha,\beta\in\mathcal{V}}\inf_{x^* \in \mathbb{S}}
\frac{\delta}{d}\sum_{i=1}^{d}|(\alpha_i-\beta_i)\cdot h_i'(x^*(i))|
\\=&\inf_{\alpha,\beta\in\mathcal{V}}\inf_{x^* \in \mathbb{S}}
\frac{2\delta}{d}\sum_{i=1}^{d}1[\alpha_i\ne\beta_i]\cdot|h_i'(x^*(i))|
\\ \geq &\inf_{\alpha,\beta\in\mathcal{V}}\inf_{x^* \in \mathbb{S}}
\frac{2\delta}{d}\sum_{i=1}^{d}1[\alpha_i\ne\beta_i]\cdot\inf_{j}{|h_j'(x^*(j))|}
\\\geq &\inf_{x^* \in \mathbb{S}}\frac{2\delta}{d}\cdot\inf_{j}{|h_j'(x^*(j))|}
\cdot\frac{d}{4}
\\=&\frac{\delta}{2}\inf_{x^* \in \mathbb{S}}\inf_{j}{|h_j'(x^*(j))|}
\end{split}
\]
where the last inequality follows from equation \eqref{defV}.
Thus
\begin{equation}\label{psilowerbound}
\psi(\mathcal{G}(\delta, h)) \geq \frac{\delta}{2}\inf_{x^* \in \mathbb{S}}\inf_{j}{|h_j'(x^*(j))|}
\end{equation}
which means that for any $\alpha, \beta \in \mathcal{V}, \alpha \ne \beta $ and $x \in \mathbb{S}$, we have
$
{\Vert \nabla g_\alpha(x)-\nabla g_\beta(x)\Vert_1} \geq \frac{\delta}{2}\inf_{x^* \in \mathbb{S}}\inf_{j}{|h_j'(x^*(j))|}
$.
Next we show that there can be at most one $\nabla g_\alpha$ in any $\ell_1$-ball centered anywhere with radius related to equation \eqref{psiDefinition}.
\begin{lemma}
\label{lemmapsi}
If
\begin{equation}\label{psiassumption}
\inf_{x^* \in \mathbb{S}}\inf_{j}{|h_j'(x^*(j))|}>0
\end{equation}
then for any $\widehat{\nabla}\in\mathbb{R}^d$ and any $x^* \in \mathbb{S}$ there can be at most one $\alpha\in\mathcal{V}$ where $g_\alpha \in \mathcal{G}(\delta, h)$ such that
\begin{equation}\label{onlyonealpha}
\Vert \widehat{\nabla}-\nabla g_\alpha(x^*) \Vert_1
\leq \frac{1}{3} \psi(\mathcal{G}(\delta, h))
\end{equation}
\end{lemma}
\begin{proof}
By contradiction, assume there exists some $\widehat{\nabla}\in\mathbb{R}^d$ and $x^* \in \mathbb{S}$ and $\alpha, \beta \in \mathcal{V}$ where $\alpha \ne \beta$ such that
$
\Vert \widehat{\nabla}-\nabla g_\alpha(x^*) \Vert_1
\leq \frac{1}{3} \psi(\mathcal{G}(\delta, h))
$
Assume that $\beta$ fulfills the same condition. Then
\[
\begin{split}
\psi(\mathcal{G}(\delta, h)) =& \inf_{a,b\in\mathcal{V}, a\ne b}\inf_{x \in \mathbb{S}}{\Vert \nabla g_a(x)-\nabla g_b(x)\Vert_1}
\\\leq&
\Vert \nabla g_\alpha(x^*) - \nabla g_\beta(x^*)\Vert _1
\\ \leq & \Vert \widehat{\nabla} - \nabla g_\alpha(x^*) \Vert_1 + \Vert \widehat{\nabla} - \nabla g_\beta(x^*)\Vert _1
\\ \leq & \frac{1}{3} \psi(\mathcal{G}(\delta, h)) + \frac{1}{3} \psi(\mathcal{G}(\delta, h))
\end{split}
\]
Therefore
$
\psi(\mathcal{G}(\delta, h)) \leq \frac{2}{3} \psi(\mathcal{G}(\delta, h))
$
which can only be true if $\psi(\mathcal{G}(\delta, h)) \leq 0$, however, from equations \eqref{psilowerbound} and \eqref{psiassumption} we get that $\psi(\mathcal{G}(\delta, h)) > 0$ thus we reach a contradiction and prove our claim.
\qedhere
\end{proof}

\subsection{Upper Bounding Probability of Estimator being Wrong}

Next, if the assumption of Lemma~\ref{lemmapsi} holds then we can further claim that if a model $\mathcal{M}_T$ can achieve a minimax error bounded as
\begin{equation}\label{riskassumption}
\EX_\phi[\epsilon(\mathcal{M}, \mathcal{G}(\delta, h), \mathbb{S}, \phi)]\leq\frac{1}{9} \psi(\mathcal{G}(\delta, h))
\end{equation}
Then we claim that model $\mathcal{M}$ can output a value $\widehat{\alpha}(\mathcal{M})$ to be the $\alpha\in\mathcal{V}$ where $\Vert \widehat{\nabla}_T - \nabla g_\alpha(x^*)\Vert_1\leq \frac{1}{3} \psi(\mathcal{G}(\delta, h))$ and if no $\alpha$ fulfills that condition then uniformly at random choose from $\mathcal{V}$. In addition, from Lemma~\ref{lemmapsi} either one valid $\alpha$ exists or none of the $\alpha$'s satisfy the condition. Thus, using Markov's inequality, we can claim that such an output from the model is wrong at most $\frac{1}{3}$ of the time.
\begin{lemma}
\label{lemmaupperbounderror}
If assumptions in Lemma~\ref{lemmapsi} hold then, if a model $\mathcal{M}$ attains an expected risk as in equation \eqref{riskassumption} then the model can construct an estimator $\widehat{\alpha}(\mathcal{M})$ to estimate the true vertex $\alpha\in\mathcal{V}$ with an error upper bounded as
\[
\begin{split}
\max_{\alpha^*\in\mathcal{V}} P_\phi[\widehat{\alpha}(\mathcal{M})\ne\alpha^*]\leq\frac{1}{3}
\end{split}
\]
which implies that if the model has an expected risk of $\EX_\phi[\epsilon(\mathcal{M}, \mathcal{G}(\delta, h), \mathbb{S}, \phi)]\leq\frac{1}{9} \psi(\mathcal{G}(\delta, h))$ then it will only fail to retrieve the correct $g_\alpha$ with probability at most $\frac{1}{3}$.
\end{lemma}
\begin{proof}
From Lemma~\ref{lemmapsi}, a maximum of one $\alpha$ exists such that inequality \eqref{onlyonealpha} is fulfilled, thus the model chooses that $\alpha$ as its estimator which means that the model is possibly incorrect when the output $\widehat{\nabla}$ from the model does not fulfill inequality \eqref{onlyonealpha} for the correct underlying $\alpha$, thus the probability of the estimator being wrong is
\[
P_\phi[\widehat{\alpha}(\mathcal{M})\ne\alpha^*] \leq P_\phi\left[\epsilon(\mathcal{M}, \mathcal{G}(\delta, h), \mathbb{S}, \phi) \geq \frac{1}{3} \psi(\mathcal{G}(\delta, h))\right]
\]
Furthermore, from Markov's inequality
\[
\begin{split}
&P_\phi\left[\epsilon(\mathcal{M}, \mathcal{G}(\delta, h), \mathbb{S}, \phi) \geq \frac{1}{3} \psi(\mathcal{G}(\delta, h))\right]
\\&\leq \frac{\EX[\epsilon(\mathcal{M}, \mathcal{G}(\delta, h), \mathbb{S}, \phi)]}{\frac{1}{3} \psi(\mathcal{G}(\delta, h))} \leq \frac{1}{3}
\end{split}
\]
Taking the maximum over $\alpha^*\in\mathcal{V}$ proves our lemma.
\qedhere
\end{proof}
\subsection{Defining an Oracle}

Now we will begin to define a specific oracle that is inspired from \cite{Agarwal00} albeit in the context of convex optimization with access to exact gradients. Let the oracle $\phi\in\mathcal{O}$ be defined as follows. At each time $t=1,\dots,T$, when the model queries position $x_t$, the oracle chooses uniformly at random a number from $i=1,\dots,d$. Then the oracle chooses $b$ at random from a Bernoulli distribution with parameter $p=\frac{1}{2}+\alpha_i\delta$ and returns $\left[b\cdot\frac{h_i(x(i))}{2}+
(1-b)\cdot\frac{-h_i(x(i))}{2}\right]$. Thus the expectation of the oracle on a function $g_\alpha \in \mathcal{G}(\delta, h)$ is
\[
\begin{split}
&\EX[\phi(x_t, g_\alpha)]
\\&= \frac{1}{d}\sum_{i=1}^{d}
{(1/2+\alpha_i\delta)\frac{h_i(x_t(i))}{2}
+(1/2-\alpha_i\delta)\frac{-h_i(x_t(i))}{2}}
\\&=\frac{\delta}{d}\sum_{i=1}^{d}
{\alpha_ih_i(x_t(i))}=g_\alpha(x_t)
\end{split}
\]

Thus the oracle is an unbiased estimator of $g_\alpha$. The uncentered second order-moment is
\[
\begin{split}
&\EX[\phi(x_t, g_\alpha)^2]
\\&= \frac{1}{d}\sum_{i=1}^{d}
{(1/2+\alpha_i\delta)\frac{h_i(x_t(i))^2}{4}+
(1/2-\alpha_i\delta)\frac{h_i(x_t(i))^2}{4}}
\\&=\frac{1}{4d}\sum_{i=1}^{d}
{h_i(x_t(i))^2}
=\frac{1}{4d}\Vert H(x_t) \Vert_2^2
\end{split}
\]
Thus, the variance fulfills
\[
\begin{split}
\Var[\phi(x_t, g_\alpha)]
=&\EX[\phi(x_t, g_\alpha)^2] - \EX[\phi(x_t, g_\alpha)]^2
\\=& \frac{1}{4d}\Vert H(x_t) \Vert_2^2 - g_\alpha(x_t)^2
\leq \frac{1}{4d}\Vert H(x_t)\Vert_2^2
\end{split}
\]
Let $\mathbb{S}\subset \mathbb{R}^d$ be the $\ell_\infty$-ball of radius $r$ such that $\forall x\in\mathbb{S}: \Vert x \Vert _\infty \leq r$. Then the variance of the oracle is upper bounded by
$
\Var[\phi(x_t, g_\alpha)]
\leq \frac{1}{4d}\Vert H(x_t)\Vert_2^2
\leq \frac{1}{4} \sup_{i\in \{1,\dots,d\}}\sup_{x\in [-r,r]}{h_i(x)^2}
$.
Since we made the condition of $\Var[\phi]\leq \sigma^2$ to hold, we set the condition on $r$ to be
$
\sup_{i\in \{1,\dots,d\}}\sup_{x\in [-r,r]}{|h_i(x)|} \leq 2 \sigma
$.

Now, information-theoretic methods can help to upper-bound the KL divergence and to find a lower-bound on the probability of retrieving $g_\alpha$. Defining some notation, let $i_t$ denote the index $i$ the oracle chose at time $t$, and $b_t$ denote the value of b the oracle chose at time $t$. Thus, the information the oracle reveals can be fully characterized by $\{(i_1,b_1),(i_2,b_2),\dots,(i_T,b_T)\}$.

\subsection{Upper Bounding KL Divergence}

Denote the distribution of the information the oracle reveals as $\mathcal{P}^T_\alpha$ and the distribution at a single time $t$ as $\mathcal{P}_\alpha$. Note that $\mathcal{P}_\alpha(i,b)=\frac{1}{d}\mathcal{P}_{\alpha_i}(b)$ since $i$ is chosen uniformly at random. Now we can find an upper bound on the KL divergence between $\mathcal{P}_\alpha^T$ and $\mathcal{P}_{\alpha'}^T, \forall \alpha, \alpha'\in\mathcal{V}, \alpha \ne \alpha'$. More formally,
\[
\begin{split}
KL(\mathcal{P}^T_\alpha||\mathcal{P}^T_{\alpha'})&=
\sum_{t=1}^{T}
KL(\mathcal{P}_\alpha(i_t,b_t)||\mathcal{P}_{\alpha'}(i_t,b_t))
\\&=\sum_{t=1}^{T}\sum_{j=1}^{d}\frac{1}{d}
KL(\mathcal{P}_{\alpha_j}(b_t)||\mathcal{P}_{\alpha'_j}(b_t))
\end{split}
\]
Each term $KL(\mathcal{P}_{\alpha_j}(b_t)||\mathcal{P}_{\alpha'_j}(b_t))$ is at most the KL divergence between two Bernoulli distributions with parameters $p=(\frac{1}{2}+\delta)$ and $q=(\frac{1}{2}-\delta)$ respectively, which is upper bounded as follows
\[
\begin{split}
KL(\mathcal{P}_{\alpha_j}(b_t)||\mathcal{P}_{\alpha'_j}(b_t))
=&\left((\frac{1}{2}+\delta)\log{\frac{(\frac{1}{2}+\delta)}{(\frac{1}{2}-\delta)}}\right) \\&+\left((\frac{1}{2}-\delta)\log{\frac{(\frac{1}{2}-\delta)}{(\frac{1}{2}+\delta)}}\right)
\\=& 2 \delta \log\left(1+\frac{4\delta}{1-2\delta}\right)
\\\leq& \frac{2 \delta \cdot 4\delta}{1-2\delta}
\end{split}
\]
Then we have $\frac{8\delta^2}{1-2\delta} \leq 16\delta^2$ when $0<\delta \leq 1/4$. Thus, if $0<\delta \leq 1/4$ then
\begin{equation}\label{KL}
KL(\mathcal{P}^T_\alpha||\mathcal{P}^T_{\alpha'}) \leq 16T\delta^2
\end{equation}

\subsection{Lower Bounding Probability of Estimator being Wrong}

Suppose that a vector $\alpha^*$ is chosen uniformly at random from $\mathcal{V}$. Then for any model $\mathcal{M}\in\mathbb{M}_T$ that makes $T$ queries to oracle $\phi$. Next we show that if $\delta\leq\frac{1}{4}$, we can make use of Fano's inequality.
\begin{lemma}
\label{lemmafano}
Any model $\mathcal{M}$ that constructs any estimator $\widehat{\alpha}(\mathcal{M})$ to estimate the true vertex $\alpha\in\mathcal{V}$ from $T$ queries attains an error lower bounded as
\[
\begin{split}
\max_{\alpha^*\in\mathcal{V}}&P_\phi[\widehat{\alpha}(\mathcal{M}_T)\ne\alpha^*]
\geq
\left\{1-\frac{16T\delta^2+\log2}{\frac{d}{2}\log(2/\sqrt{e})}\right\}
\end{split}
\]
\end{lemma}
\begin{proof}
Using Fano's inequality \cite{Yu97} and equations \eqref{cardinalityV} and \eqref{KL}, we have that
\[
\begin{split}
\max_{\alpha^*\in\mathcal{V}}&P_\phi[\widehat{\alpha}(\mathcal{M}_T)\ne\alpha^*]
\\&\geq\left\{1-\frac
{\max_{\alpha,\beta\in\mathcal{V}}\{KL(\mathcal{P}^T_\alpha||\mathcal{P}^T_{\beta})\}+\log2}
{\log|\mathcal{V}|}\right\}
\\&\geq
\left\{1-\frac
{16T\delta^2+\log2}
{\log(2/\sqrt{e})^{d/2}}\right\}
\\&\geq
\left\{1-\frac
{16T\delta^2+\log2}
{\frac{d}{2}\log(2/\sqrt{e})}\right\}
\end{split}
\]
\qedhere
\end{proof}

\subsection{Concluding the Proof of Theorem~\ref{hyperplanes}}

Now by setting $h_i(x)=h(x)=x$ and for some $\delta > 0$, we analyze the set $G(\delta, h)$ which can be interpreted as a set of hyperplanes in $\mathbb{R}^d$, each dimension with a slope of $+\delta \text{ or } -\delta$. Now we make sure our conditions hold by witnessing that
$
\delta\inf_{x^* \in \mathbb{S}}\inf_{j}{|h_j'(x^*(j))|}=\delta\inf_{x^* \in \mathbb{S}}\inf_{j}{|1|}=\delta>0
$.
Thus, the conditions for Lemma~\ref{lemmapsi} and Lemma~\ref{lemmaupperbounderror} hold. By setting $r=2\sigma$ for the radius of the $\ell_\infty$-ball of the input space we have
$
\sup_{x\in [-r,r]}{|h(x)|} = \sup_{x\in [-2\sigma,2\sigma]}{|x|} = 2 \sigma \leq 2 \sigma
$.
Thus, $\Var[\phi]\leq \sigma^2$ holds.
Then, let $k=\inf_{x^* \in \mathbb{S}}\inf_{j}{|h_j'(x^*(j))|}=1$ and by setting $\epsilon=k\delta/18$, if a model $\mathcal{M}$ achieves
\[
\EX[\epsilon^*(\mathcal{M}_T, \mathcal{G}(\delta, h), \mathbb{S}, \phi)]\leq\frac{\delta}{18} k=\epsilon
\]
and from Lemma~\ref{lemmaupperbounderror}, we have that $\max_{\alpha^*\in\mathcal{V}}P_\phi[\widehat{\alpha}(\mathcal{M}_T)\ne\alpha^*]\leq\frac{1}{3}$.
In addition, from Lemma~\ref{lemmafano} we have that $\max_{\alpha^*\in\mathcal{V}}P_\phi[\widehat{\alpha}(\mathcal{M}_T)\ne\alpha^*]\geq
\left\{1-\frac
{16T(18\epsilon/k)^2+\log2}
{\frac{d}{2}\log(2/\sqrt{e})}\right\}$. Combining the two terms we get
\[
\begin{split}
&\frac{1}{3}\geq
\left\{1-\frac
{16T(18\epsilon/k)^2+\log2}
{\frac{d}{2}\log(2/\sqrt{e})}\right\}
\\ \Rightarrow \;\;\; &
\epsilon
\geq \sqrt{\frac{\log(2/\sqrt{e})\cdot d-3\log2}{324\cdot 3\cdot 16T}\cdot k^2}
\\ \Rightarrow \;\;\; &
T=\Omega(\frac{d}{\epsilon^2})
\text{ and }
\epsilon \geq c\sqrt{\frac{d}{T}}
\end{split}
\]

Hence, the main theorem is proven, since we conclude that
\[
\begin{split}
\epsilon^*(\mathcal{F}, \mathbb{S}, \phi)
=&\inf_{\mathcal{M}\in\mathbb{M}_T} \sup_{f\in\mathcal{F}}
\sup_{x^*\in\mathbb{S}}
\EX_\phi[\Vert \widehat{\nabla}_T - \nabla f(x^*) \Vert _ 1 ]
\\\geq& \sqrt{\frac{\log(2/\sqrt{e})\cdot d-3\log2}{324\cdot 3\cdot 16T}\cdot k^2}
\end{split}
\]

\section{Proof of Theorem~\ref{FDM}: Rate of Finite Differences for a Bounded-Variance Oracle}
We begin by stating the output estimate from the finite difference method (FDM) and state the $\ell_1$-error. Given any oracle $\phi(x, f)$ that outputs noisy unbiased values for the $d$-dimensional function $f$ such that $E[\phi(x, f)]=f(x)$ and $\Var[\phi(x, f)]\leq \sigma^2$. Assume that the oracle can be queried up to $T$ times and that the function $f$ is three times differentiable. Applying FDM on $\phi(x, f)$ gives the output $\nabla \widehat{f}(x^*)_\mathcal{FDM}=[\widehat{\nabla}_1, \widehat{\nabla}_2,\dots,\widehat{\nabla}_d]^\top$ and for each dimension $i$, the value $\widehat{\nabla}_i$ is calculated as follows:
$
\widehat{\nabla}_i = \frac{1}{T/2d}
\sum_{j=1}^{T/2d}
{\left(\frac{\phi^{(j)}(x^*+he_i, f)-\phi^{(j)}(x^*-he_i, f)}{2h}\right)}
$.
The error of FDM given any valid oracle $\phi$ and point $x^*\in\mathbb{S}$ is given as follows
\[
\begin{split}
&\epsilon(\mathcal{FDM}, \mathcal{F}, \mathbb{S}, \phi) = \EX_\phi [\Vert \nabla \widehat{f}(x^*)_\mathcal{FDM} - \nabla f(x^*) \Vert _ 1 ]
\\&=
\EX_\phi
\left[\sum_{i=1}^{d}
\left\vert
\frac{1}{T/2d}\sum_{j=1}^{T/2d}
{V_i^{(j)}}
- \nabla f(x^*)_i
\right\vert \right]
\end{split}
\]
where $V_i^{(j)} := \frac{\phi^{(j)}(x^*+he_i, f)-\phi^{(j)}(x^*-he_i, f)}{2h}$, and $h$ is some constant that the method chooses. We are interested on finding an upper bound for $\epsilon(\mathcal{FDM}, \mathcal{F}, \mathbb{S}, \phi)$ for any oracle $\phi$ and any function $f$ that is three times differentiable and any point $x^*\in\mathbb{S}$.
Recall that $f^{iii}$ represents the third derivative of $f$ with respect to the $i$-th variable.
By Taylor series expansion up to the third order we know that
\[
\frac{f(x^*+he_i)-f(x^*-he_i)}{2h}=\nabla f(x^*)_i+\frac{h^2}{6}f^{iii}(\xi_i)
\]
where $e_i = [0,\dots,0,1,0,\dots,0]^\top$ and the entry with value $1$ is at the $i$-th index. Note that $\xi_i \in \mathbb{R}^d$ belongs to the line segment between $x^*-he_i$ and $x^*+he_i$, and only depends on the value of $x^*$ and the function $f$.

Let $W_i^{(j)} := V_i^{(j)} - \frac{h^2}{6}f^{iii}(\xi_i) = \frac{\phi^{(j)}(x^*+he_i, f)-\phi^{(j)}(x^*-he_i, f)}{2h} - \frac{h^2}{6}f^{iii}(\xi_i)$.
Since $E[\phi(x, f)]=f(x)$, we have that $E_\phi[W_i^{(j)}] = \nabla f(x^*)_i$.
Similarly, since $\Var[\phi(x, f)]\leq \sigma^2$, we have that $\Var[W_i^{(j)}] = \frac{1}{4h^2}(\Var[\phi(x^*+he_i, f)] + \Var[\phi(x^*-he_i, f)]) \leq \frac{\sigma^2}{2h^2}$.
Let $n=T/2d$.
From Chebyshev's inequality we get
\[
\begin{split}
&P\left(\left|
\frac{1}{n}\sum_{j=1}^{n}
W_i^{(j)}-\nabla f(x^*)_i\right|
>z\right)
\\&\leq \min\left(\frac{\sum_{j=1}^{n}\Var[W_i^{(j)}]}{n^2 z^2}, 1\right)
\\&\leq \min\left(\frac{\sigma^2}{2nh^2z^2}, 1\right)
\end{split}
\]
Let $Z=\left|\frac{1}{T/2d}\sum_{j=1}^{T/2d} W_i^{(j)}-\nabla f(x^*)_i\right|$.
We then apply the layer-cake representation for the expected value of the non-negative random variable $Z$ and we get
\[
\begin{split}
E[Z]&=\int_{0}^{\infty}P(Z>z)dz \leq \int_{0}^{\infty}\min\left(\frac{\sigma^2}{2nh^2z^2}, 1\right)dz
\\&\leq \int_{0}^{\sigma/(h\sqrt{2n})}1dz
+ \frac{\sigma^2}{2nh^2} \int_{\sigma/(h\sqrt{2n})}^{\infty} \frac{1}{z^2} dz
\\&= \frac{\sigma}{h\sqrt{2n}} + \frac{\sigma^2}{2nh^2}\cdot\frac{h\sqrt{2n}}{\sigma} = \frac{\sigma\sqrt{2}}{h\sqrt{n}}
= \frac{2\sigma\sqrt{d}}{h\sqrt{T}}
\end{split}
\]
Thus we get the upper bound on the expectation, and we can take the $f^{iii}(\xi_i)$ term out of the expectation as follows.
We have
\[
\begin{split}
&E_\phi\left[\left|\frac{1}{T/2d}\sum_{j=1}^{T/2d}\left(V_i^{(j)} - \frac{h^2}{6}f^{iii}(\xi_i)\right)-\nabla f(x^*)_i\right|\right]
\\&\leq \frac{2\sigma\sqrt{d}}{h\sqrt{T}}
\\ \Rightarrow \;\;\; &
E_\phi\left[\left|\frac{1}{T/2d}\sum_{j=1}^{T/2d}V_i^{(j)}-\nabla f(x^*)_i\right|\right]
\\&\leq \frac{2\sigma\sqrt{d}}{h\sqrt{T}}+\frac{h^2}{6} \vert f^{iii}(\xi_i) \vert
\leq \frac{2\sigma\sqrt{d}}{h\sqrt{T}}+\frac{h^2}{6} K
\end{split}
\]
where $K=\sup_{i\in \{1,\dots,d\}}\sup_{x\in\mathbb{S}}|f^{iii}(x)|$.
Now we can use this upper bound on the expectation of the estimated error on a single dimension $i$ to calculate the upper bound of the $\ell_1$-error of FDM.
We have
\[
\begin{split}
&\epsilon(\mathcal{FDM}, \mathcal{F}, \mathbb{S}, \phi) 
\\&= \EX_\phi [\Vert \nabla \widehat{f}(x^*)_\mathcal{FDM} - \nabla f(x^*) \Vert _ 1 ]
\\&=
\EX_\phi
\left[\sum_{i=1}^{d}
{\left\vert \frac{1}{T/2d}\sum_{j=1}^{T/2d}
V_i^{(j)} - \nabla f(x^*)_i \right\vert} \right]
\\&=
\sum_{i=1}^{d}
\EX_\phi
\left[
{\left\vert
\frac{1}{T/2d}\sum_{j=1}^{T/2d}
V_i^{(j)} - \nabla f(x^*)_i
\right\vert} \right]
\\&\leq \frac{2\sigma d\sqrt{d}}{h\sqrt{T}}+\frac{h^2}{6}dK
\end{split}
\]
Now we can use a value $h>0$ that minimizes the $\ell_1$-error of FDM which is chosen to be
$
h=\sqrt[3]{\frac{2\sigma d\sqrt{d}/\sqrt{T}}{2Kd/6}}
$
provided that $K>0$.
Thus the upperbound on the $\ell_1$-error now becomes
\[
\begin{split}
\epsilon(\mathcal{FDM}, \mathcal{F}, \mathbb{S}, \phi)
\leq& \frac{2\sigma d\sqrt{d}}{h\sqrt{T}}+\frac{h^2}{6}dK
\\=&d^{(4/3)} \sqrt[3]{\frac{1}{T}} \sqrt[3]{\frac{9\sigma^2K}{2}}
\end{split}
\]
Furthermore, if $K=0$, we then use a value $h>0$ that minimizes the $\ell_1$-error of FDM while ignoring the $K$ term thus $h$ is chosen to be that largest possible value it can take. Such a value, theoretically speaking, is the minimum distance from the point we want to estimate $x^*$ to the region $\mathbb{S}$, which is a constant and for clarity of presentation we will call $h_r$. More formally,
\[
\begin{split}
h_r^+ &= \inf_{i\in \{1,\dots,d\}} \sup_{c\in\mathbb{R}} c\cdot 1[(x^*+ce_i)\in\mathbb{S}]
\\
h_r^- &= \inf_{i\in \{1,\dots,d\}} \sup_{c\in\mathbb{R}} c\cdot 1[(x^*-ce_i)\in\mathbb{S}]
\\
h_r &= \min{\{h_r^+, h_r^-\}}
\end{split}
\]
Thus the upperbound when $K=0$ becomes
\[
\epsilon(\mathcal{FDM}, \mathcal{F}, \mathbb{S}, \phi)\leq d^{(3/2)}\frac{1}{\sqrt{T}}\frac{2\sigma}{h_r}
\]

\section{Proof of Theorem~\ref{FDMgaussian}: Rate of Finite Differences for a Gaussian Oracle}

Recall that the proof of Theorem~\ref{FDM} considers a general oracle $\phi$ where $E[\phi(x, f)]=f(x)$, and $\Var[\phi(x, f)]\leq \sigma^2$, and then uses Chebyshev's inequality.
Here we will avoid using any inequality.

Recall $V_i^{(j)} := \frac{\phi^{(j)}(x^*+he_i, f)-\phi^{(j)}(x^*-he_i, f)}{2h}$.
Now, for a tight result, assume a Gaussian oracle $\phi$, that is, $\phi(x, f)$ is Gaussian distributed with mean $f(x)$ and variance $\sigma^2$.
Formally, $\phi(x, f) \sim \mathcal{N}(f(x),\sigma^2)$ and thus $\EX[\phi(x, f)] = f(x)$ and $\Var[\phi(x, f)] = \sigma^2$.

First, consider the case $K=0$.
For constants $c_1,\dots,c_d,c'_1,\dots,c'_d>0$, consider a family of functions of the form:
\[
f(x) = \sum_{i=1}^d \left( K \frac{x_i^3}{6} + c_i \frac{x_i^2}{2} + c'_i x_i \right)
\]
By Taylor series expansion up to the third order and since $f^{iii}(x)=K$ for every $x$, we have
\[
\frac{f(x^*+he_i)-f(x^*-he_i)}{2h}=\nabla f(x^*)_i+\frac{h^2}{6}K
\]
Note that $\EX[V_i^{(j)}] = \frac{f(x^*+he_i, f)-f(x^*-he_i, f)}{2h}=\nabla f(x^*)_i+\frac{h^2}{6}K$.
Note that $\Var[V_i^{(j)}] = \frac{1}{4h^2}(\Var[\phi(x^*+he_i, f)] + \Var[\phi(x^*-he_i, f)]) = \frac{\sigma^2}{2h^2}$.
Therefore
\[
\EX\left[\frac{1}{T/2d}\sum_{j=1}^{T/2d}V_i^{(j)}\right] = \nabla f(x^*)_i+\frac{h^2}{6}K
\]
and
\[
\Var\left[\frac{1}{T/2d}\sum_{j=1}^{T/2d}V_i^{(j)}\right] = \frac{\sigma^2 d}{h^2 T}
\]
For a Gaussian variable $Z \sim \mathcal{N}(\mu,s^2)$ we know that $\EX[|Z|] = s \exp(-\mu^2/(2 s^2)) \sqrt{2/\pi} + \mu \, \erf(\mu/(\sqrt{2}s))$ where $\erf$ is the Gauss error function, thus
\[
\begin{split}
&E_\phi\left[\left|\frac{1}{T/2d}\sum_{j=1}^{T/2d}V_i^{(j)}-\nabla f(x^*)_i\right|\right]
\\&= \frac{\sigma\sqrt{d}}{h\sqrt{T}} \exp\hns\left(-\frac{K^2 h^6 T}{72 \sigma^2 d}\right) \sqrt{\frac{2}{\pi}}+\frac{h^2}{6} K \, \erf\hns\left(\frac{K h^3 \sqrt{T}}{\sqrt{72} \sigma d}\right)
\end{split}
\]
Thus
\begin{align} \label{eq:FDM_tight}
&\epsilon(\mathcal{FDM}, \mathcal{F}, \mathbb{S}, \phi)
\nonumber\\&= \EX_\phi [\Vert \nabla \widehat{f}(x^*)_\mathcal{FDM} - \nabla f(x^*) \Vert _ 1 ]
\nonumber\\&=
\EX_\phi
\left[\sum_{i=1}^{d}
{\left\vert \frac{1}{T/2d}\sum_{j=1}^{T/2d}
V_i^{(j)} - \nabla f(x^*)_i \right\vert} \right]
\nonumber\\&=
\sum_{i=1}^{d}
\EX_\phi
\left[
{\left\vert
\frac{1}{T/2d}\sum_{j=1}^{T/2d}
V_i^{(j)} - \nabla f(x^*)_i
\right\vert} \right]
\nonumber\\&=\frac{\sigma d\sqrt{d}}{h\sqrt{T}} \exp\hns\left(-\frac{K^2 h^6 T}{72 \sigma^2 d}\right) \sqrt{\frac{2}{\pi}}+\frac{h^2}{6} d K \, \erf\hns\left(\frac{K h^3 \sqrt{T}}{\sqrt{72} \sigma d}\right)
\end{align}
Unfortunately, it is difficult to find a value $h>0$ that minimizes the $\ell_1$-error of FDM above.
Since for $z \geq 0$ we have $\exp(-z) \in [0,1]$ and $\erf(z) \in [0,1]$, we will consider minimizing the expression
\[
\frac{\sigma d\sqrt{d}}{h\sqrt{T}} \sqrt{\frac{2}{\pi}}+\frac{h^2}{6} d K
\]
which leads to choose
$
h=\frac{d^{1/6} \sigma^{1/3}}{K^{1/3} T^{1/6}} \left(\frac{18}{\pi}\right)^{1/6}
$.
Plugging this back into equation \eqref{eq:FDM_tight}, we get
\[
\begin{split}
&\epsilon(\mathcal{FDM}, \mathcal{F}, \mathbb{S}, \phi)
\\& =\frac{\sigma d\sqrt{d}}{h\sqrt{T}} \exp\hns\left(-\frac{K^2 h^6 T}{72 \sigma^2 d}\right) \sqrt{\frac{2}{\pi}}+\frac{h^2}{6} d K \, \erf\hns\left(\frac{K h^3 \sqrt{T}}{\sqrt{72} \sigma d}\right)
\\& = d^{(4/3)} \sqrt[3]{\frac{1}{T}} \sqrt[3]{\frac{\sigma^2K}{12\pi}} \left(2\exp\hns\left(-\frac{\pi}{4}\right)+\erf\hns\left(\frac{1}{2\sqrt{\pi}}\right)\right)
\end{split}
\]

Second, consider the case $K>0$.
From the definition of $K=\sup_{i\in \{1,\dots,d\}}\sup_{x\in\mathbb{S}}|f^{iii}(x)|$, we conclude that $f^{iii}(\xi_i) = 0$.
By Taylor series expansion up to the third order and since $f^{iii}(\xi_i) = 0$, we have
\[
\frac{f(x^*+he_i)-f(x^*-he_i)}{2h}=\nabla f(x^*)_i
\]
Note that $\EX[V_i^{(j)}] = \frac{f(x^*+he_i, f)-f(x^*-he_i, f)}{2h}=\nabla f(x^*)_i$.
Note that $\Var[V_i^{(j)}] = \frac{1}{4h^2}(\Var[\phi(x^*+he_i, f)] + \Var[\phi(x^*-he_i, f)]) = \frac{\sigma^2}{2h^2}$.
Therefore
\[
\EX\left[\frac{1}{T/2d}\sum_{j=1}^{T/2d}V_i^{(j)}\right] = \nabla f(x^*)_i
\]
and
\[
\Var\left[\frac{1}{T/2d}\sum_{j=1}^{T/2d}V_i^{(j)}\right] = \frac{\sigma^2 d}{h^2 T}
\]
For a Gaussian variable $Z \sim \mathcal{N}(0,s^2)$ we know that $\EX[|Z|] = s \sqrt{2/\pi}$, thus
\[
\begin{split}
E_\phi\left[\left|\frac{1}{T/2d}\sum_{j=1}^{T/2d}V_i^{(j)}-\nabla f(x^*)_i\right|\right]
&= \frac{\sigma\sqrt{d}}{h\sqrt{T}} \sqrt{\frac{2}{\pi}}
\end{split}
\]
Thus
\[
\begin{split}
&\epsilon(\mathcal{FDM}, \mathcal{F}, \mathbb{S}, \phi) 
\\&= \EX_\phi [\Vert \nabla \widehat{f}(x^*)_\mathcal{FDM} - \nabla f(x^*) \Vert _ 1 ]
\\&=
\EX_\phi
\left[\sum_{i=1}^{d}
{\left\vert \frac{1}{T/2d}\sum_{j=1}^{T/2d}
V_i^{(j)} - \nabla f(x^*)_i \right\vert} \right]
\\&=
\sum_{i=1}^{d}
\EX_\phi
\left[
{\left\vert
\frac{1}{T/2d}\sum_{j=1}^{T/2d}
V_i^{(j)} - \nabla f(x^*)_i
\right\vert} \right]
\\&=\frac{\sigma d\sqrt{d}}{h\sqrt{T}} \sqrt{\frac{2}{\pi}}
\end{split}
\]
We then use a value $h>0$ that minimizes the $\ell_1$-error of FDM thus $h$ is chosen to be that largest possible value it can take. Such a value, theoretically speaking, is the minimum distance from the point we want to estimate $x^*$ to the region $\mathbb{S}$, which is a constant and for clarity of presentation we will call $h_r$. More formally,
\[
\begin{split}
h_r^+ &= \inf_{i\in \{1,\dots,d\}} \sup_{c\in\mathbb{R}} c\cdot 1[(x^*+ce_i)\in\mathbb{S}]
\\
h_r^- &= \inf_{i\in \{1,\dots,d\}} \sup_{c\in\mathbb{R}} c\cdot 1[(x^*-ce_i)\in\mathbb{S}]
\\
h_r &= \min{\{h_r^+, h_r^-\}}
\end{split}
\]
Finally
\[
\epsilon(\mathcal{FDM}, \mathcal{F}, \mathbb{S}, \phi)= d^{(3/2)}\frac{1}{\sqrt{T}}\frac{\sigma}{h_r} \sqrt{\frac{2}{\pi}}
\]

\section{Concluding Remarks}

There are several ways to extend our results. A new method of approximation could be analyzed with a faster convergence compared to the finite difference method. One can also analyze the estimation of second order derivatives, i.e., the Hessian matrix, which can be useful in optimization methods. Finally, one can also consider the estimation of integrals which can be useful in the approximation of expectations.

\bibliographystyle{abbrv}
\bibliography{references}

\clearpage
\appendix
\onecolumn

\section{Details about the Taylor expansion in the proof of Theorem~\ref{FDM} and Theorem~\ref{FDMgaussian}}

Let $h>0$ and let $e_i = [0,\dots,0,1,0,\dots,0]^\top$ where the entry with value $1$ is at the $i$-th index.
We first express both $f(x^*+he_i)$ and $f(x^*-he_i)$ using the Taylor expansion up to the third order, to get
\[
\begin{split}
f(x^*+he_i)&=f(x^*)+h\nabla f(x^*)_i+\frac{h^2}{2}f^{ii}(x^*)+\frac{h^3}{6}f^{iii}(\xi_i^+)
\\f(x^*-he_i)&=f(x^*)-h\nabla f(x^*)_i+\frac{h^2}{2}f^{ii}(x^*)-\frac{h^3}{6}f^{iii}(\xi_i^-)
\end{split}
\]
where $\nabla f$ is the gradient of $f$, and $f^{ii}$ and $f^{iii}$ represent the second and third derivative of $f$ with respect to the $i$-th variable, respectively.
Here $\xi_i^+$ belongs to the line segment between $x^*$ and $x^*+he_i$, and $\xi_i^-$ belongs to the line segment between $x^*$ and $x^*-he_i$.
Subtracting the two equations above and dividing by $2h$, we get
\[
\frac{f(x^*+he_i)-f(x^*-he_i)}{2h}=\nabla f(x^*)_i+\frac{h^2}{6}f^{iii}(\xi_i)
\]
where $\xi_i \in \mathbb{R}^d$ belongs to the line segment between $x^*-he_i$ and $x^*+he_i$.

\end{document}